\documentclass[journal,twoside,web]{ieeecolor}
\usepackage{generic}
\usepackage{cite}
\usepackage{amsmath,amssymb,amsfonts}
\usepackage[amsmath,thmmarks]{ntheorem}
\usepackage{graphicx}
\usepackage{booktabs}
\usepackage{url}
\usepackage{hyperref}
\hypersetup{hidelinks=true}
\usepackage{textcomp}

\newtheorem{proposition}{Proposition}
\newtheorem{corollary}{Corollary}
\theoremstyle{plain}
\newtheorem{remark}{Remark}

\begin{document}

\title{Loss Invariance Determines What Concept Layers Encode: Volume Grounding in Echocardiography}

\author{Hyunkyung Han and Min Jung Kim
\thanks{Manuscript submitted \today.
This work was supported by Institute of Information \& communications Technology
Planning \& Evaluation (IITP) grant funded by the Korea government(MSIT)
(\mbox{RS-2026-25586094}).}
\thanks{H. Han is with the Department of Integrative Medicine, Yonsei University
College of Medicine, Seoul 06273, Republic of Korea
(e-mail: sthan1@yonsei.ac.kr).}
\thanks{M. J. Kim is with the Department of Integrative Medicine, Yonsei
University College of Medicine, Seoul 06273, and with the Department of
Radiology, Research Institute of Radiologic Science, Yonsei University College of
Medicine, Seoul 03722, Republic of Korea (e-mail: MINES@yuhs.ac).}
\thanks{ORCID: H. Han, 0009-0006-2672-384X; M. J. Kim, 0000-0003-4949-1237.}
\thanks{\textit{(Corresponding author: Min Jung Kim.)}}}

\maketitle

\begin{abstract}
\textit{Objective:} Concept bottleneck models route prediction through
interpretable intermediate variables, and their validity is normally judged by
how accurately those variables are predicted. We ask whether that judgement is
sufficient, using left ventricular volumes as the concepts underlying ejection
fraction estimation from echocardiographic video.
\textit{Methods:} A video transformer encoder was trained on a publicly
available echocardiography dataset. End-systolic and end-diastolic volumes
formed a concept layer from which ejection fraction was computed analytically,
with no residual path to the output. We compared training under an ejection
fraction objective alone against training with additional supervision of the
volumes in millilitres, and evaluated both on 1276 held-out studies.
\textit{Results:} The concept bottleneck did not increase ejection fraction
error relative to direct regression, at 6.89 against 7.13 mean absolute error.
Without volume supervision, however, the spread of predicted volumes collapsed
to 0.1 millilitres against reference spreads of 35.7 and 45.7 millilitres, while
correlation was partly preserved. We show that this follows from an invariance
property of the objective: ejection fraction is a ratio and is unchanged when
both volumes are rescaled, so the loss determines the concept layer only up to
scale. Supervision in absolute units reduced volume error from 89.8 to 25.8
millilitres at a cost of 0.4 in ejection fraction error.
\textit{Conclusion:} Concept accuracy alone can conceal a concept layer that
carries no physical scale.
\textit{Significance:} Interpretable intermediate variables in clinical models
should be validated against the invariance structure of the training objective,
not only against prediction accuracy.
\end{abstract}


\section{Introduction}
\label{sec:introduction}
\IEEEPARstart{C}{oncept} bottleneck models (CBMs) route a prediction through an
intermediate layer of human-interpretable quantities, so that the model's output
can be read as a function of variables a clinician already uses~\cite{koh2020}.
In medical imaging this is an attractive contract: the intermediate layer is not
a post-hoc explanation but a load-bearing part of the computation, and it can in
principle be audited, corrected, or overridden at inspection time.

The contract holds only if the concept layer means what its labels say. In
practice, validity is assessed by how accurately the layer predicts the
concepts, typically as an error or correlation against reference annotations. We
argue that this measure is not sufficient, and more specifically that what a
concept layer \emph{can} represent is fixed in advance by the invariance
structure of the training objective, independently of how well it is optimized.

Echocardiographic ejection fraction (EF) makes the point concrete. EF is defined
as a ratio of end-diastolic and end-systolic volumes and is therefore unchanged
if both volumes are multiplied by a common factor. An objective built on EF
alone places no requirement on the concept layer to express volumes in physical
units. This is not a technicality for cardiology: end-diastolic volume is itself
a clinical variable, used to assess ventricular dilation and remodeling, and a
concept layer that yields a correct EF while reporting volumes with no physical
scale would mislead precisely the reader for whom the bottleneck was built.

Related work has approached information loss in concept layers empirically.
Concept leakage describes unintended information passing through concept
activations~\cite{mahinpei2021,margeloiu2021,havasi2022}, and concept
incompleteness has motivated architectures that add residual pathways or
auxiliary concepts to recover what the layer omits~\cite{pcbm2023,rescbm2024};
a recent survey reviews both lines~\cite{knab2026}. In medical imaging,
concept bottlenecks have been built on radiological
lexicons~\cite{bunnell2024}, and the trustworthiness of concept evidence has
been examined in terms of where activations arise rather than what they
encode~\cite{huang2024,sgcbm2026}, in line with broader calls for structured
validation of clinical models~\cite{kondylakis2026}. Closest to our position is
work on reasoning shortcuts in neuro-symbolic predictors, which studies concepts
that satisfy the training objective while encoding unintended
semantics~\cite{marconato2025}, and which derives conditions under which the
concepts and the inference layer are identifiable, with intended semantics
defined up to a permutation and an element-wise invertible
transformation~\cite{bortolotti2025}. Two things separate that setting from
ours. Their analysis concerns discrete concepts and a finite label set, and the
multiplicity it counts is combinatorial; the degeneracy we describe is a
continuous one-parameter family arising from a ratio-valued target. More
importantly, an element-wise invertible transformation is there admitted as part
of the intended semantics, whereas for a concept whose meaning is a physical
quantity it is precisely the failure: a volume reported in arbitrary units is
correct up to an invertible rescaling and clinically unusable. The standard
equivalence relation is therefore too coarse for quantitative clinical concepts.
These lines treat the
omission as a deficiency to be patched, or the misplacement of evidence as the
thing to be corrected. Our position is different: we show that \emph{which}
properties can be omitted is determined by the invariance group of the loss, so
that the appropriate remedy is not added capacity but supervision that breaks
the invariance.

This work makes three contributions.
\begin{enumerate}
\item We show empirically that a volume concept layer trained under an EF
objective retains part of the ordering among studies while losing physical scale
almost entirely, with predicted volume spread collapsing to well under one
percent of the reference spread.
\item We formalize the cause as an invariance property of the objective
(Proposition~\ref{prop:invariance}), which predicts that no
invariance-preserving regularizer can recover scale. We report a design of our
own that fell into exactly this failure and was not detectable from concept
accuracy.
\item We verify the prescription that follows---supervision in absolute
units---and quantify its cost, which is an increase of 0.4 in EF mean absolute
error, together with an optimization instability it introduces and a practical
remedy.
\end{enumerate}

The encoder used here was introduced in our earlier study of attribution
faithfulness in echocardiographic video~\cite{p1}; the research question and the
contributions of the present work are distinct.

\section{Methods}

\subsection{Data}
We use the EchoNet-Dynamic dataset~\cite{ouyang2020}, publicly distributed by
Stanford University at \url{https://echonet.github.io/dynamic/}.
All data are publicly available; no institutional review board approval was
required for this secondary analysis of a de-identified public dataset.

We adopt the standard partition of 7,465 training, 1,288 validation and 1,277
test videos; one test video is excluded by the loader, and volume annotations
are unavailable for a further six studies (five in training, one in test),
giving 7,460 / 1,288 / 1,276 evaluable cases. Reference volumes were
$43.7 \pm 35.7$~mL (end-systolic, ESV) and $91.6 \pm 45.7$~mL
(end-diastolic, EDV).

\subsection{Clip sampling and encoder}
Each study contributes a single clip of 32 frames taken with a stride of two,
spanning 62 consecutive frames of the source recording. At the acquisition rate
of this dataset that span covers a complete beat for essentially every study,
whereas a stride of one would truncate the beat in a substantial fraction of
them. The distinction matters more here than it would for a direct regressor:
the two concepts are read off the extremes of a per-frame trace, so a clip that
omits either extreme leaves the concept layer unmeasurable rather than merely
noisy. The same sampling is applied in training, validation and test.

The backbone is a VideoMAE encoder~\cite{p1} operating on $112 \times 112$
frames with $16 \times 16$ patches and a tubelet depth of two. Masked
autoencoding pre-training ran for 200 epochs with 60\% of tubelets hidden and
batches of 256; ejection fraction fine-tuning then ran for 50 epochs under AdamW
with a learning rate of $10^{-4}$, weight decay 0.05 and batches of 32.
Validation loss selected epoch 43, and that one checkpoint initialized every
condition below, so the comparisons that follow differ only in the objective.

\subsection{Concept bottleneck}
The concept layer consists of two clinically defined quantities, end-systolic
and end-diastolic volume, from which ejection fraction follows analytically as
\begin{equation}
\widehat{\mathrm{EF}} = 100 \cdot
\frac{\hat{V}_{\mathrm{ED}} - \hat{V}_{\mathrm{ES}}}{\hat{V}_{\mathrm{ED}}}.
\label{eq:ef}
\end{equation}
Prediction is routed entirely through the concept layer: no residual path
connects the encoder to the EF output, so any information the concept layer
fails to carry is unavailable downstream.

The encoder emits $T=16$ temporal tokens per study (32 sampled frames at a
tubelet size of 2).
A head shared across tokens maps each token embedding to a scalar volume:
layer normalization, a linear map $768 \to 384$, a GELU nonlinearity, a linear
map $384 \to 1$, and a softplus that constrains the output to be positive. No
pooling precedes the head, so the resulting per-token volume trace
$\hat{v}_1, \dots, \hat{v}_T$ retains full temporal resolution. The two concepts
are taken as $\hat{V}_{\mathrm{ED}} = \max_t \hat{v}_t$ and
$\hat{V}_{\mathrm{ES}} = \min_t \hat{v}_t$.
The model therefore selects the cardiac phase itself rather than being given
annotated frames. This is why the temporal assumption stated in
Section~\ref{sec:temporal} must be verified rather than assumed: unless the
maximum and minimum fall in the diastolic and systolic portions of the clip
respectively, the two outputs are not the concepts they are named after.
The encoder was fine-tuned jointly with the concept head in all conditions. The
positivity constraint is worth noting for what follows: when the scale of the
concept layer is unconstrained, the head cannot collapse to an arbitrary
constant but only toward small positive values, which is what we observe in
Section~\ref{sec:results}.

The training objective is
\begin{equation}
\mathcal{L} = \mathcal{L}_{\mathrm{EF}} + \mu\,\mathcal{L}_{\mathrm{vol}},
\label{eq:loss}
\end{equation}
where $\mathcal{L}_{\mathrm{EF}}$ is the $L_1$ error of \eqref{eq:ef} and
$\mathcal{L}_{\mathrm{vol}}$ is the $L_1$ error of both volumes in millilitres.
Setting $\mu = 0$ gives the EF-only condition.

\subsubsection{A degenerate concept loss}
An earlier version of this design regularized the concept layer toward the
correct volume \emph{ratio} rather than the correct volumes. Writing
$\hat{\rho} = \hat{V}_{\mathrm{ES}} / \hat{V}_{\mathrm{ED}}$ for the predicted
ratio, that term penalized $|\hat{\rho} - \rho|$ with the target
$\rho = 1 - \mathrm{EF}$, which is the identity obtained by rearranging
\eqref{eq:ef}. The difficulty is that the same rearrangement applies to the
prediction: $\hat{\rho} = 1 - \widehat{\mathrm{EF}}$ holds identically, so
\begin{equation}
|\hat{\rho} - \rho|
= \bigl|(1-\widehat{\mathrm{EF}}) - (1-\mathrm{EF})\bigr|
= \bigl|\widehat{\mathrm{EF}} - \mathrm{EF}\bigr| ,
\label{eq:degenerate}
\end{equation}
with ejection fraction expressed as a fraction of unity. The regularizer was
therefore not merely similar to $\mathcal{L}_{\mathrm{EF}}$ but \emph{equal} to
it, contributing a constant multiple of the same gradient. It is invariant under
the scale group $G$ and supplies no direction not already present in the ejection
fraction objective, which is exactly the situation described in
Remark~\ref{rem:invariant}.

We report this because the failure is not visible from concept accuracy alone:
the normalized concept error was low throughout training while the absolute
scale of the concept layer was never constrained at any point. An evaluation
protocol that scores concepts in normalized units would have reported the model
as well grounded. The normalized condition reported in
Section~\ref{sec:results} is a different construction: it rescales the concept
loss by a batch statistic but retains absolute magnitude information, and does
restore grounding.

\subsection{Volume grounding}
Concept heads were additionally supervised in physical units (mL) with weight
$\mu \in \{3, 10\}$. A normalized variant, in which the concept loss is scaled
by the batch mean volume, was included as a control.

\subsection{Verification of the temporal assumption}
\label{sec:temporal}
The concept heads must attend to the correct phase of the cardiac cycle for any
subsequent claim about scale to be meaningful. Across the test set, the EDV head
placed its maximum within the first five temporal tokens in 97.5\% of studies,
the ESV head placed its minimum within the last eleven tokens in 99\% of
studies, and the maximum preceded the minimum in 97.7\% of studies. Failures of
scale reported below are therefore not attributable to temporal misalignment.

\subsection{Evaluation}
We report mean absolute error (MAE) with bootstrap 95\% confidence intervals and
Pearson correlation on the 1,276 test studies. For the concept layer we
additionally report the standard deviation of predicted volumes, which is the
quantity of interest for scale. A constant predictor attains an EF MAE of 9.097
and serves as a lower reference.

\section{Theory}

\begin{proposition}
\label{prop:invariance}
Let the output map $g$ be invariant under a group $G$ acting on the concept
space, i.e. $g(\tau(c)) = g(c)$ for all $\tau \in G$ and all $c$. Then the
training loss determines the concept layer only up to a $G$-orbit.
\end{proposition}

\noindent\textit{Proof sketch:}
For any concept representation $c$ and any $\tau \in G$, the predictions
$g(c)$ and $g(\tau(c))$ coincide, hence the two representations incur identical
loss. Loss minimization therefore cannot distinguish members of the same orbit.

\begin{corollary}
EF depends on ESV and EDV only through their ratio and is thus invariant under
the scale group $G = \{c \mapsto \lambda c : \lambda > 0\}$. An EF objective
alone does not determine the absolute scale of the concept layer; supervision in
physical units is necessary.
\end{corollary}

The proposition itself is elementary and belongs to the same family as
identifiability results in representation learning, where a model class is shown
to be determined only up to an equivalence class of
transformations~\cite{locatello2019}. Our claim is not that the argument is new
mathematics, but that it has not been used to constrain supervision design in
concept bottlenecks, where concept validity is instead assessed by prediction
accuracy against annotations.

\begin{remark}
\label{rem:invariant}
A regularizer that is itself $G$-invariant cannot narrow the orbit and amounts
to a reparameterization.
\end{remark}

\section{Results}
\label{sec:results}

\subsection{The concept bottleneck does not cost EF accuracy}
Table~\ref{tab:ef} reports EF error. The concept bottleneck did not increase EF
error relative to direct regression (6.891 vs.\ 7.129). The confidence intervals
overlap, so we do not claim superiority; the result is nonetheless in the
opposite direction from the common expectation that routing a prediction through
a low-dimensional interpretable layer incurs an accuracy penalty. The supervised
video baseline R(2+1)D-18 is included as an upper reference and is not a
comparison target for this study.

\begin{table}[t]
\caption{Ejection Fraction Error on the Test Set ($n=1{,}276$)}
\label{tab:ef}
\centering
\setlength{\tabcolsep}{4pt}
\begin{tabular}{lccc}
\toprule
Model & MAE & 95\% CI & $r$ \\
\midrule
R(2+1)D-18 (reference) & 4.814 & --- & 0.848 \\
\textbf{CBM, EF only} & \textbf{6.891} & [6.502, 7.269] & 0.623 \\
Direct regression & 7.129 & [6.735, 7.531] & 0.581 \\
CBM, grounded $\mu=3$ & 7.295 & [6.906, 7.702] & 0.580 \\
CBM, normalized & 7.574 & [7.169, 7.983] & 0.544 \\
CBM, grounded $\mu=10$ & 7.698 & [7.287, 8.113] & 0.510 \\
Constant predictor & 9.097 & --- & --- \\
\bottomrule
\end{tabular}
\end{table}

\subsection{The concept layer loses physical scale}
Table~\ref{tab:concept} reports the concept layer. Without volume supervision,
the standard deviation of predicted volumes collapsed to 0.1~mL against
reference standard deviations of 35.7 and 45.7~mL. The EDV correlation was
$-0.219$; at this predicted spread the correlation is numerically unstable and
we treat its sign as uninformative rather than as evidence of anti-correlation.
The ESV correlation remained positive at 0.626. The concept layer therefore
retains part of the ordering while losing the unit, which is precisely the
behaviour predicted by Proposition~\ref{prop:invariance}. Fig.~\ref{fig:scatter}
shows the effect directly: the ungrounded predictions form a horizontal band near
zero regardless of the reference volume.

\begin{figure*}[t]
\centering
\includegraphics[width=0.72\linewidth]{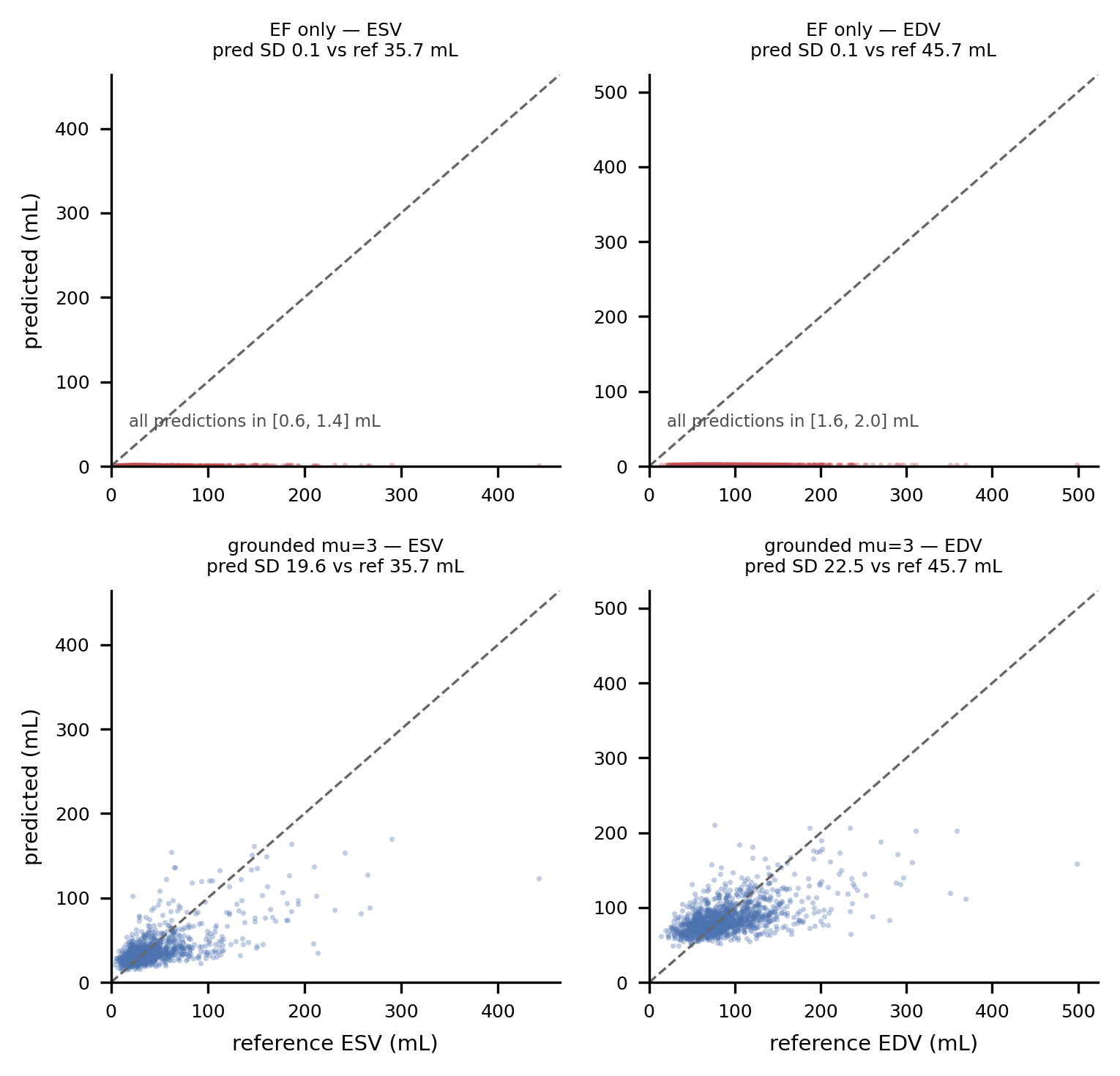}
\caption{Predicted against reference volumes on the test set ($n=1{,}276$).
\textbf{Top:} without volume supervision the concept layer produces almost
constant volumes (predicted SD 0.1~mL against reference SD 35.7 and 45.7~mL;
across all 1{,}276 studies the end-systolic predictions span 0.6--1.4~mL and the
end-diastolic predictions 1.6--2.0~mL), collapsing onto a narrow band along the
horizontal axis; the softplus output
constrains that collapse to small positive values rather than to an arbitrary
constant. Both rows share axis limits. \textbf{Bottom:} supervision in
millilitres restores the spread
(19.6 and 22.5~mL) and aligns the predictions with the identity line, though the
remaining slope shows that compression is reduced rather than removed. Dashed
line is identity.}
\label{fig:scatter}
\end{figure*}

\begin{table*}[t]
\caption{Concept Layer Grounding on the Test Set ($n=1{,}276$). Predicted SD Is the Quantity of Interest for Scale.}
\label{tab:concept}
\centering
\setlength{\tabcolsep}{6pt}
\begin{tabular}{lcccccc}
\toprule
& \multicolumn{3}{c}{End-systolic volume} & \multicolumn{3}{c}{End-diastolic volume} \\
\cmidrule(lr){2-4}\cmidrule(lr){5-7}
Condition & MAE (mL) & $r$ & pred.\ SD & MAE (mL) & $r$ & pred.\ SD \\
\midrule
Reference & --- & --- & \textbf{35.7} & --- & --- & \textbf{45.7} \\
CBM, EF only & 42.96 & 0.626 & \textbf{0.1} & 89.79 & $-0.219$ & \textbf{0.1} \\
CBM, grounded $\mu=3$ & 16.03 & 0.663 & 19.6 & 25.83 & 0.569 & 22.5 \\
CBM, normalized & 16.59 & 0.600 & 18.1 & 26.50 & 0.523 & 21.5 \\
CBM, grounded $\mu=10$ & 16.60 & 0.623 & 17.9 & 26.50 & 0.547 & 22.5 \\
\bottomrule
\end{tabular}
\end{table*}

\subsection{Grounding restores scale at a cost in EF accuracy}
Volume supervision reduced concept error by roughly two thirds (ESV
$42.96 \rightarrow 16.03$; EDV $89.79 \rightarrow 25.83$) and raised predicted
spread from 0.1 to 19.6 and 22.5~mL. EF error increased by 0.404
($6.891 \rightarrow 7.295$). This is a direct consequence of the corollary:
supervision that breaks the scale invariance also constrains the optimization of
a scale-invariant target.

\subsection{Increasing the grounding weight does not improve concept validity}
Concept errors were unchanged between $\mu=3$ and $\mu=10$ (ESV 16.03 vs.\
16.60; EDV 25.83 vs.\ 26.50) while EF error increased. Two runs at matched
learning rate differed by 0.13 in EF MAE, against a $\mu=3$ to $\mu=10$
difference of approximately 0.26. The direction is consistent but its magnitude
is on the order of twice the seed-to-seed variation, so we do not claim a
monotone relationship. The substantive finding is that a larger grounding weight
buys no additional concept validity.

\subsection{Optimization stability}
\begin{figure}[t]
\centering
\includegraphics[width=\linewidth]{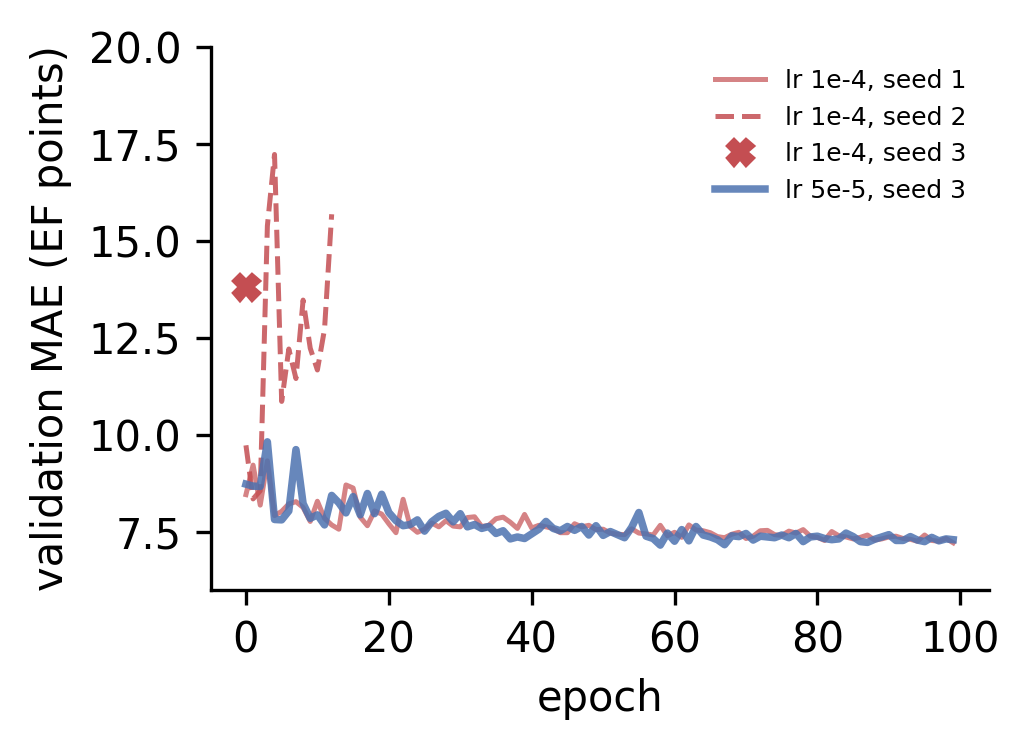}
\caption{Validation error during volume-grounded training. At a learning rate of
$10^{-4}$ two of three seeds diverged and were terminated; halving the learning
rate removed the instability without changing the final error materially.}
\label{fig:divergence}
\end{figure}

At a learning rate of $10^{-4}$, two of three seeds diverged: one reached a
validation MAE of 15.38 at epoch 3 and had risen to 17.22 by epoch 4 without
recovering, and one was already at 13.81 at epoch 0. Both were terminated once
divergence was evident, at epochs 13 and 1 respectively; the surviving seed
completed 100 epochs at a final validation MAE of 7.22. Halving the learning
rate to $5 \times 10^{-5}$ eliminated divergence, and that run completed 100
epochs at 7.30. Physical-unit supervision changes the loss scale substantially,
and we recommend reducing the learning rate accordingly
(Fig.~\ref{fig:divergence}).

\subsection{Compression is not resolved}
Even with grounding, predicted spread reached only 22.5~mL against a reference
of 45.7~mL, and this value was essentially identical across all three grounded
conditions. The residual compression is therefore not a matter of loss
weighting.

\section{Discussion}

\subsection{What is new}
Information loss in concept layers has been treated as an empirical deficiency
to be patched with residual pathways. We instead show that \emph{what} can be
lost is fixed in advance by the invariance group of the objective, and that the
corresponding remedy is invariance-breaking supervision rather than added
capacity.

It is worth separating this from concept completeness, which asks whether the
concept set is a sufficient statistic for the target. Our concept set is
complete in that sense: the two volumes determine EF exactly, and no
task-relevant information is missing from the bottleneck. The layer is
nonetheless uninterpretable in physical terms. Completeness and physical
grounding are therefore distinct axes, and a concept set can satisfy the first
while failing the second. It also follows that an evaluation reporting concept
error in normalized units cannot detect the failure, consistent with reports
that accuracy-driven evaluation can mask divergence from the intended reasoning
process~\cite{ezarlenga2023metrics}.

A second literature treats the opposite failure. Concept layers are widely
reported to encode more than their assigned concept, and both diagnostic
metrics~\cite{ezarlenga2023metrics} and information-bottleneck
regularizers~\cite{almudevar2025} have been proposed to remove the excess. That
mechanism is unavailable here: the label is an analytic function of two scalars,
so there is no capacity in which task-relevant nuisance could be carried past the
concept layer. The failure we report is the complement -- the layer carries too
little rather than too much -- and it persists in a bottleneck narrow enough that
leakage cannot account for it. The two lines do agree on one point: constraining
the concept layer, in either direction, is paid for in task
accuracy~\cite{almudevar2025}, which is what we observe when grounding is added.

Recent work on grounding concept bottlenecks in medical imaging has been
concerned with space. Spatially grounded models use coarse lesion delineations
to constrain where concept evidence may arise~\cite{sgcbm2026}, and related work
examines whether concept activations localize to the regions they
name~\cite{huang2024}. Such constraints answer \emph{where} a concept is read
from. They do not, however, address the degeneracy studied here: no restriction
on the spatial support of an activation can break an algebraic invariance of the
objective, because every member of the scale orbit is compatible with any given
spatial pattern. The grounding we require is dimensional rather than spatial --
not where the model looks, but in what units the concept layer reports. The two
are complementary constraints on a concept layer, and satisfying one says
nothing about the other.

\subsection{A scale hypothesis for the residual compression}
\label{sec:scale}
In a separate measurement set with recorded calibration, the physical pixel
spacing could be recovered exactly within a study (coefficient of variation
$2 \times 10^{-10}$) but varied between studies from 0.011 to 0.061~cm/px.
Decomposing this variation attributes 32\% to differences in stored resolution
and a coefficient of variation of 15\% to acquisition depth. Because
preprocessing resizes all inputs to a fixed grid, the physical scale is removed
from the input. The residual compression is consistent with this account,
although we do not establish it as the cause.

\subsection{Limitations}
All conditions in Tables~\ref{tab:ef} and~\ref{tab:concept} were trained for 100
epochs from the same fine-tuned encoder checkpoint, so the comparisons reported
here are at matched optimization budget.
An earlier round of experiments compared conditions at unequal budgets, and we
do not report those numbers; small differences in ejection fraction error
between grounded variants are within the range we observed across seeds and
should not be read as effects of the grounding weight.

Three limitations remain. The residual compression of predicted volume spread is
unexplained, and the account we offer in Section~\ref{sec:scale} is a hypothesis
rather than a demonstrated cause. Results come from a single dataset acquired at
one institution, and generalization to other scanners and acquisition protocols
is untested. Each configuration was run with two to three seeds, which is
sufficient to bound the differences we do claim but not to characterize the
optimization landscape.

\subsection{Generalization}
The argument is not specific to echocardiography. Any concept vocabulary that is
scale-invariant by construction---for example shape and margin descriptors---is
predicted to exhibit the same behaviour under an objective that does not
reference absolute size.

\section{Conclusion}
Routing echocardiographic ejection fraction through a volume concept layer costs
nothing in accuracy, but the concept layer that results carries almost no
physical scale. This is not an optimization failure to be patched with capacity:
it is determined in advance by the fact that ejection fraction is a ratio, and
therefore that the objective constrains the concept layer only up to a common
rescaling. Supervision in absolute units breaks the invariance and restores
grounding, at a measurable and modest cost in ejection fraction accuracy. The
broader implication for interpretable clinical models is that a concept layer
should be validated against what the objective is capable of determining, and
not against concept accuracy alone.

\section*{Declaration of Generative AI and AI-Assisted Technologies}
During the preparation of this work the authors used a commercially available
large language model to improve the language of the manuscript and to assist
with drafting and debugging analysis code. The tool was not used to generate
scientific claims, to interpret results, or to produce references: all citations
were verified by the authors against the original publications, all numerical
results were computed by code written and reviewed by the authors, and all
figures and tables were generated from source data by the authors. The authors
reviewed and edited the content as needed and take full responsibility for the
content of the article.

\section*{Data Availability}
This study uses the publicly available EchoNet-Dynamic
dataset~\cite{ouyang2020}, at \url{https://echonet.github.io/dynamic/}; access
requires acceptance of the Stanford AIMI data use agreement. Code to reproduce
the experiments will be released upon publication.

\end{document}